\documentclass[twoside]{article}

\usepackage{amssymb, graphicx}
\usepackage{bm}
\usepackage{url}
\usepackage{nicefrac}
\usepackage{caption}
\usepackage{float}
\usepackage{dblfloatfix}
\usepackage{amsfonts}
\usepackage{amsthm}
\usepackage{amsmath}
\usepackage{xcolor}
\usepackage{enumitem}
\usepackage{booktabs}
\usepackage{hhline}
\usepackage{multirow}
\usepackage{rotating}
\usepackage{titlesec}
\usepackage{fontawesome}
\usepackage{mdframed}
\newtheorem{theorem}{Theorem}

\usepackage{algorithm}
\usepackage{algorithmic}

\DeclareMathOperator{\Tr}{Tr}

\definecolor{mydarkblue}{rgb}{0,0.08,0.45}
\usepackage[colorlinks=true,
    linkcolor=mydarkblue,
    citecolor=mydarkblue,
    filecolor=mydarkblue,
    urlcolor=mydarkblue]{hyperref}
\usepackage[accepted]{arxiv}
\usepackage[round]{natbib}

\begin{document}

\twocolumn[

\aistatstitle{Fenchel Lifted Networks: \\
A Lagrange Relaxation of Neural Network Training}

\aistatsauthor{Fangda Gu* \And Armin Askari* \And  Laurent El Ghaoui}

\aistatsaddress{UC Berkeley \And  UC Berkeley \And UC Berkeley} 
\aistatsaddress{}
]

\begin{abstract}
  Despite the recent successes of deep neural networks, the corresponding training problem remains highly non-convex and difficult to optimize. Classes of models have been proposed that introduce greater structure to the objective function at the cost of lifting the dimension of the problem. However, these lifted methods sometimes perform poorly compared to traditional neural networks. In this paper, we introduce a new class of lifted models, Fenchel lifted networks, that enjoy the same benefits as previous lifted models, without suffering a degradation in performance over classical networks. Our model represents activation functions as equivalent biconvex constraints and uses Lagrange Multipliers to arrive at a rigorous lower bound of the traditional neural network training problem. This model is efficiently trained using block-coordinate descent and is parallelizable across data points and/or layers. We compare our model against standard fully connected and convolutional networks and show that we are able to match or beat their performance.
\end{abstract}

\section{Introduction}

Deep neural networks (DNNs) have become the preferred model for supervised learning tasks after their success in various fields of research. However, due to their highly non-convex nature, DNNs pose a difficult problem during training time; the optimization landscape consists of many saddle points and local minima which make the trained model generalize poorly \citep{entropy2016, dauphin2014}. This has motivated regularization schemes such as weight decay \citep{krogh1992}, batch normalization \citep{ioffe2015}, and dropout \citep{dropout} so that the solutions generalize better to the test data. 

In spite of this, backprop used along with stochastic gradient descent (SGD) or similar variants like Adam \citep{kingma2015adam} suffer from a variety of problems. One of the most notable problems is the vanishing gradient problem which slows down gradient-based methods during training time. Several approaches have been proposed to deal with the problem; for example, the introduction of rectified linear units (ReLU). However, the problem persists. For a discussion on the limitations of backprop and SGD, we direct the reader to Section 2.1 of  \citet{taylor2016training}.

One approach to deal with this problem is to introduce auxiliary variables that increase the dimension of the problem. In doing so, the training problem decomposes into multiple, local sub-problems which can be solved efficiently without using SGD or Adam; in particular, the methods of choice have been block coordinate descent (BCD) \citep{askari2018, lau2018proximal, zhang2017convergent, pmlr-v33-carreira-perpinan14} and the alternating direction method of multipliers (ADMM) \citep{taylor2016training, zhang2016efficient}. By lifting the dimension of the problem, these models avoid many of the problems DNNs face during training time. In addition, lifting offers the possibility of penalizing directly the added variables, which opens up interesting avenues into the interpretability and robustness of the network.

While these methods, which we refer to as ``lifted'' models for the remainder of the paper, offer an alternative to the original problem with some added benefits, they have their limitations. Most notably, traditional DNNs are still able to outperform these methods in spite of the difficult optimization landscape. As well, most of the methods are unable to operate in an online manner or adapt to continually changing data sets which is prevalent in most reinforcement learning settings \citep{Sutton:1998:IRL:551283}. Finally, by introducing auxiliary variables, the dimensionality of the problem greatly increases, making these methods very difficult to train with limited computational resources.

\subsection{Paper contribution}
To address the problems listed above, we propose Fenchel lifted networks, a biconvex formulation for deep learning based on Fenchel's duality theorem that can be optimized using BCD. We show that our method is a  rigorous \emph{lower bound} for the learning problem and admits a natural batching scheme to adapt to changing data sets and settings with limited computational power. We compare our method against other lifted models and against traditional fully connected and convolutional neural networks. We show that we are able to outperform the former and that we can compete with or even outperform the latter.

\paragraph{Paper outline.} 
In Section \ref{sec:relatedwork}, we give a brief overview of related works on lifted models. In Section \ref{sec:background} we introduce the notation for the remainder of the paper. Section \ref{sec:fn} introduces Fenchel lifted networks, their variants and discusses how to train these models using BCD. Section \ref{sec:exp} compares the proposed method against fully connected and convolutional networks on MNIST and CIFAR-10.

\section{Related Work} \label{sec:relatedwork}




\paragraph{Lifted methods}
Related works that lift the dimension of the training problem are primarily optimized using BCD or ADMM. These methods have experienced recent success due to their ability to exploit the structure of the problem by first converting the constrained optimization problem into an unconstrained one and then solving the resulting sub-problems in parallel. They do this by relaxing the network constraints and introducing penalties into the objective function. The two main ways of introducing penalties into the objective function are either using quadratic penalties \citep{Sutskever:2013:IIM:3042817.3043064,taylor2016training, lau2018proximal} or using equivalent representations of the activation functions \citep{askari2018, zhang2017convergent}. 

As a result, these formulations have many advantages over the traditional training problem, giving superior performance in some specific network structures \citep{pmlr-v33-carreira-perpinan14, zhang2017convergent}. These methods also enjoy great potential to parallelize as shown by \citet{taylor2016training}. However, there has been little evidence showing that these methods can compete with traditional DNNs which shadows the nice structure these formulations bring about. 

An early example of auxiliary variables being introduced into the training problem is the method of auxiliary coordinates (MAC) by \citet{pmlr-v33-carreira-perpinan14} which uses quadratic penalties to enforce network constraints. They test their method on auto encoders and show that their method is able to outperform SGD. Followup work by \citet{carreira2016parmac, taylor2016training} demonstrate the huge potential for parallelizing these methods. \citet{lau2018proximal} gives some convergence guarantee on a modified problem. 

Another class of models that lift the dimension of the problem do so by representing activation functions in equivalent formulations. \citet{Negiar2017, askari2018, zhang2017convergent, li2019lifted} explore the structure of activation functions and use $\arg \min$ maps to represent activation functions. In particular, \citet{askari2018} show how a strictly monotone activation function can be seen as the $\arg \min$ of a specific optimization problem. Just as with quadratic penalties, this formulation of the problem still performs poorly compared to traditional neural networks.

\section{Background and Notation} \label{sec:background}
\paragraph{Feedforward neural networks.} We are given an input data matrix of $m$ data points $ X = [x_1, x_2, ..., x_m] \in \mathbb{R}^{n\times m} $ and a response matrix $ Y \in \mathbb{R}^{p\times m}$. We consider the supervised learning problem involving a neural network with $ L \geq 1 $ hidden layers. The neural network produces a prediction $\hat{Y} \in \mathbb{R}^{p\times m} $ with the feed forward recursion $ \hat{Y} = W_L X_{L} + b_L \bm{1}^\top$ given below
\begin{align} \label{fnnr_old}
    X_{l+1} = \phi_l(W_l X_l + b_l \bm{1}^\top), ~~ l = 0, \hdots, L-1.
\end{align}

where $\phi_l, l = 0,\hdots,L$ are the activation functions that act column-wise on a matrix input, $\textbf{1} \in \mathbb{R}^m$ is a vector of ones, and $W_l \in \mathbb{R}^{p_{l+1} \times p_l}$ and $b_l \in \mathbb{R}^{p_{l+1}}$ are the weight matrices and bias vectors respectively. Here $p_l$ is the number of output values for a single data point (i.e., hidden nodes) at layer $l$ with $p_0 = n$ and $p_{L+1} = p$. Without loss of generality, we can remove $b_l \textbf{1}^\top$ by adding an extra column to $W_l$ and a row of ones to $X_l$. Then \eqref{fnnr_old} simplifies to
\begin{align} \label{fnnr}
    X_{l+1} = \phi_l(W_l X_l), ~~ l = 0, \hdots, L-1.
\end{align}

In the case of fully connected networks, $\phi_l$ is typically sigmoidal activation functions or ReLUs. In the case of Convolutional Neural Networks (CNNs), the recursion can accommodate convolutions and pooling operations in conjunction with an activation. For classification tasks, we typically apply a softmax function after applying an affine transformation to $X_L$.

The initial value for the recursion is $X_0 = X$ and $X_l\in \mathbb{R}^{p_l \times m}, \; l = 0,\hdots, L$. We refer to the collections $(W_l)_{l=0}^L$ and $(X_l)_{l=1}^{L}$ as the $W$ and $X$-variables respectively.

The weights are obtained by solving the following constrained optimization problem
\begin{align} \label{nn}
	\min_{(W_l)_{l=0}^{L}, (X_l)_{l=1}^{L}} & \: {\cal L}(Y,W_L X_L) + \sum^L_{l=0} \rho_l \pi_l(W_l) \notag \\ 
	\mbox{s.t.}  \;\; &X_{l+1} = \phi_l(W_l X_l),\; l = 0,\hdots, L-1 \notag \\ 
    \; &X_0 = X 
\end{align}

Here, ${\cal L} $ is a loss function, $\rho \in \mathbb{R}^{L+1}_+$ is a hyper-parameter vector, and $\pi_l$'s are penalty functions used for regularizing weights, controlling network structures, etc. In \eqref{nn}, optimizing over the $X$-variables is trivial; we simply apply the recursion \eqref{fnnr} and solve the resulting unconstrained problem using SGD or Adam. After optimizing over the weights and biases, we obtain a prediction $\hat{Y}$ for the test data $X$ by passing $X$ through the recursion \eqref{fnnr} one layer at a time.


\paragraph{Our model.}
We develop a family of models where we approximate the recursion constraints \eqref{fnnr} via penalties. We use the $\arg \min$ maps from \citet{askari2018} to create a biconvex formulation that can be trained efficiently using BCD and show that our model is a lower bound of \eqref{nn}. Furthermore, we show how our method can naturally be batched to ease computational requirements and improve the performance.

\section{Fenchel lifted networks}\label{sec:fn}
In this section, we introduce Fenchel lifted networks. We begin by showing that for a certain class of activation functions, we can equivalently represent them as biconvex constraints. We then dualize these constraints and construct a lower bound for the original training problem. We show how our lower bound can naturally be batched and how it can be trained efficiently using BCD.

\subsection{Activations as bi-convex constraints}
In this section, we show how to convert the equality constraints of \eqref{nn} into inequalities which we dualize to arrive at a \textit{relaxation} (lower bound) of the problem. In particular, this lower bound is biconvex in the $W$-variables and $X$-variables. We make the following assumption on the activation functions $\phi_l$.

\begin{quote}
\textbf{BC Condition} The activation function $\phi : \mathbb{R}^p \rightarrow \mathbb{R}^q$ satisfies the BC condition if there exists a \emph{biconvex} function $B_\phi : \mathbb{R}^{p} \times \mathbb{R}^{p} \rightarrow \mathbb{R}_+$, such that 
\[
v = \phi(u) \Longleftrightarrow B_\phi(v,u) \leq 0.
\]
\end{quote}

We now state and prove a result that is at the crux of Fenchel lifted networks.

\begin{theorem}\label{thm:bc}
Assume $\phi: \mathbb{R} \rightarrow \mathbb{R}$ is continuous, strictly monotone and that $0 \in \text{range}(\phi)$ or $0 \in \text{domain}(\phi)$. Then $\phi$ satisfies the BC condition.
\end{theorem}

\begin{proof}
Without loss of generality, $\phi$ is strictly increasing. Thus it is invertible and there exists $\phi^{-1}$ such that $u = \phi^{-1}(v)$ for $v \in \text{range}(\phi)$ which implies $v = \phi(u)$. Now, define $F: \mathbb{R}^{p} \rightarrow \mathbb{R}$ as
\begin{align*}
F(v) &:= \int_{z}^{v} \phi^{-1}(\xi) \; d\xi 
\end{align*}

where $z \in \text{range}(\phi)$ and is either $0$ or satisfies $\phi^{-1}(z) = 0$. Then we have
\begin{align} \label{equ:fencheldivergence}
F^\ast (u) &=  \int_{\phi^{-1}(z)}^{u} \phi(\eta) \; d\eta \notag \\
B(v,u) &= F(v) + F^\ast(u) - uv
\end{align}

where $F^\ast$ is the Fenchel conjugate of $F$. By the Fenchel-Young inequality, $B(v,u) \geq 0$ with equality if and only if 
\begin{align*}
    v^\ast = \arg\max_v \; uv - F(v) \; : \; v \in \text{range}(\phi)
\end{align*}

By construction, $v^\ast = \phi(u)$. Note furthermore since $\phi$ is continuous and strictly increasing, so is $\phi^{-1}$ on its domain, and thus $F, F^\ast$ are convex. It follows that $B(v,u)$ is a biconvex function of $(u,v)$. 

We simply need to prove that $F^\ast(u)$ above is indeed the Fenchel conjugate of $F$. By definition of the Fenchel conjugate we have that
\begin{align*}
    F^\ast (u) = \max_v \; u v - F(v) \; : \; v \in \text{range}(\phi)
\end{align*}

It is easy to see that $v^\ast = \phi(u)$. Thus
\begin{align*}
    F^\ast(u) &= u\phi(u) - F(\phi(u)) \\
    &=  u \phi(u) - \int_{z}^{\phi (u)} \phi^{-1} (\xi) \; d\xi \\
    &= \int_{z}^{\phi(u)} \xi \dfrac{d}{d \xi} \phi^{-1} (\xi) \; d\xi \\
    &=  \int_{\phi^{-1}(z)}^{u} \phi(\eta) \; d \eta
\end{align*}
where the third equality is a consequence of integration by parts, and the fourth equality we make the subsitution $\eta = \phi^{-1} (\xi)$
\end{proof}

Note that Theorem \ref{thm:bc} implies that activation functions such as sigmoid and tanh can be equivalently written as a biconvex constraint. Although the ReLU is not strictly monotone, we can simply restrict the inverse to the domain $\mathbb{R}_+$; specifically, for $\phi(x) = \max(0,x)$ define
\[
\phi^{-1} (z) = \left\{ \begin{array}{ll} +\infty & \mbox{if } z < 0, \\ z & \mbox{if } z \ge 0,
\end{array} \right.
\]

Then, we can rewrite the ReLU function as the equivalent set of biconvex constraint
\[
    v = \max(0,u) \Longleftrightarrow \left\{ 
    \begin{array}{ll} \dfrac{1}{2} v^2 + \dfrac{1}{2} u_+^2 - uv \leq 0\\
    v \geq 0
    \end{array} \right.
\]
where $u_+ = \max(0,u)$. This implies
\begin{equation} \label{eq:reluBC}
        B_\phi(v,u) =
        \left\{ \begin{array}{ll} \dfrac{1}{2} v^2 + \dfrac{1}{2} u_+^2 - uv  & \text{if} \; v \geq 0 \\ 
    +\infty & \text{otherwise}
    \end{array} \right.
\end{equation}

Despite the non-smoothness of $u_+$, for fixed $u$ or fixed $v$,  \eqref{eq:reluBC} belongs in $C^1$ -- that is, it has continuous first derivative and can be optimized using first order methods. We can trivially extend the result of Theorem \ref{thm:bc} for matrix inputs: for matrices $U,V \in \mathbb{R}^{p \times q}$, we have
\begin{align*}
    B_\phi(V,U) = \sum_{i,j} B_\phi(V_{ij},U_{ij}).
\end{align*}



\subsection{Lifted Fenchel model}

Assuming the activation functions of \eqref{nn} satisfy the hypothesis of Theorem \ref{thm:bc}, we can reformulate the learning problem equivalently as
\begin{align} \label{fnn}
	\min_{(W_l)_{l=0}^{L}, (X_l)_{l=1}^{L}} & \: {\cal L}(Y,W_L X_L) + \sum^L_{l=0} \rho_l \pi_l(W_l) \notag\\ 
	 \mbox{s.t.} \; \; B_l(X_{l+1}&, W_lX_l) \le 0, \; l = 0,\hdots, L-1 \notag \\
    X_0 = X,
\end{align}
where $B_l$ is the short-hand notation of $B_{\phi_l}$. We now dualize the inequality constraints and obtain the lower bound of the standard problem \eqref{nn} via Lagrange relaxation
\begin{align} \label{fnnp2}
	G(\lambda) := &\min_{(W_l)_{l=0}^{L}, (X_l)_{l=1}^{L}} \: {\cal L}(Y,W_LX_L) + \sum^L_{l=0}\rho_l \pi_l(W_l)\notag \\
	&+ \sum^{L-1}_{l=0} \lambda_l B_l(X_{l+1}, W_lX_l) \notag \\ 
	\mbox{s.t.} & \: X_0 = X ,
\end{align}
where $\lambda_l \geq 0$ are the Lagrange multipliers. The maximum lower bound can be achieved by solving the dual problem
\begin{align}
    p^\ast \geq d^\ast = \max_{\lambda \geq 0} G(\lambda)
\end{align}
where $p^\ast$ is the optimal value of \eqref{nn}. Note if all our activation functions are ReLUs, we must also include the constraint $X_l \geq 0$ in the training problem as a consequence of \eqref{eq:reluBC}. Although the new model introduces $L$ new parameters (the Lagrange multipliers), we can show that using variable scaling we can reduce this to only \emph{one} hyperparameter (for details, see \ref{apx:varscale}). The learning problem then becomes
\begin{align} \label{fnnp}
	G(\lambda) := &\min_{(W_l)_{l=0}^{L}, (X_l)_{l=1}^{L}} \: {\cal L}(Y,W_LX_L) + \sum^L_{l=0}\rho_l \pi_l(W_l)\notag \\
	&+ \lambda \sum^{L-1}_{l=0} B_l(X_{l+1}, W_lX_l) \notag \\ 
	\mbox{s.t.} & \: X_0 = X .
\end{align}

In a regression setting where the data is generated by a one layer network, we are able to provide global convergence guarantees of the above model (for details, see \ref{apx:global}).

\paragraph{Comparison with other methods.} 
For ReLU activations, $B(v,u)$ as in \eqref{eq:reluBC} differs from the penalty terms introduced in previous works. In \citet{askari2018, zhang2017convergent} they set $B(v,u) = \|v-u\|_2^2$ and in \citet{taylor2016training, pmlr-v33-carreira-perpinan14} they set $B(v,u) = \|v - u_+\|_2^2$. Note that $B(v,u)$ in the latter is not biconvex. While the $B(v,u)$ in the former is biconvex, it does not perform well at test time.
\cite{li2019lifted} set $B(v,u)$ based on a proximal operator that is similar to the BC condition. 

\paragraph{Convolutional model.}
Our model can naturally accommodate average pooling and convolution operations found in CNNs, since they are linear operations. We can rewrite $W_lX_l$ as $W_l \ast X_l$ where $\ast$ denotes the convolution operator and write Pool($X$) to denote the average pooling operator on $X$. Then, for example, the sequence Conv $\rightarrow$ Activation can be represented via the constraint
\begin{align}
    B_l(X_{l+1}, W_l \ast X_l) \le 0,
\end{align}
while the sequence Pool $\rightarrow$ Conv $\rightarrow$ Activation can be represented as
\begin{align}
    B_l(X_{l+1}, W_l \ast \text{Pool}(X_l)) \le 0.
\end{align}
Note that the pooling operation changes the dimension of the matrix.

\subsection{Prediction rule.}
In previous works that reinterpret activation functions as $\arg\min$ maps \citep{askari2018, zhang2017convergent}, the prediction at test time is defined as the solution to the optimization problem below 
\begin{align} \label{lpr}
	\hat{y} = \arg&\min_{y, (x_l)} \: {\cal L}(y ,W_Lx_L) + \lambda \sum^{L-1}_{l=0} B_l(x_{l+1}, W_lx_l)\notag \\
	\mbox{s.t.} & \: x_0 = x, 
\end{align}
where $x_0$ is test data point, $\hat{y}$ is the predicted value, and $x_l$, $l = 1,\hdots,L$ are the intermediate representations we optimize over. Note if $\mathcal{L}$ is a mean squared error, applying the traditional feed-forward rule gives an optimal solution to \eqref{lpr}. We find empirically that applying the standard feed-forward rule works well, even with a cross-entropy loss. 





\subsection{Batched model} \label{sec:batchedflnn}
The models discussed in the introduction usually require the entire data set to be loaded into memory which may be infeasible for very large data sets or for data sets that are continually changing. We can circumvent this issue by batching the model. By sequentially loading a part of the data set into memory and optimizing the network parameters, we are able to train the network with limited computational resources. Formally, the batched model is
\begin{align} \label{bfnnp}
	\displaystyle\min_{(W_l)_{l=0}^{L}, (X_l)_{l=1}^{L}}  \: {\cal L}(Y,W_LX_L) + \sum^L_{l=0}\rho_l \pi_l(W_l) \notag \\
	+ \lambda \sum^{L-1}_{l=0} B_l(W_lX_l, X_{l+1}) + \sum^{L}_{l=0} \gamma_l \|W_l - W_{l}^0\|_F^2  \notag \\ 
	\mbox{s.t. } X_0 = X ,
\end{align}

where $X_0$ contains only a batch of data points instead of the complete data set. The additional term in the objective $ \gamma_l  \|W_l - W_{l}^0\|_F^2  , ~~ l = 0,\hdots, L$ is introduced to moderate the change of the $W$-variables between subsequent batches; here $W_{l}^0$ represents the optimal $W$ variables from the previous batch and $\gamma \in \mathbb{R}_+^{L+1}$ is a hyperparameter vector. The $X$-variables are reinitialized each batch by feeding the new batch forward through the equivalent standard neural network. 

\subsection{Block-coordinate descent algorithm}
The model \eqref{fnnp} satisfies the following properties:
\begin{itemize}
\item For fixed $W$-variables, and fixed variables $(X_j)_{j \ne l}$, the problem is convex in $X_l$, and is decomposable across data points.
\item For fixed $X$-variables, the problem is convex in the $W$-variables, and is decomposable across layers, and data points.
\end{itemize}

The non-batched and batched Fenchel lifted network are trained using block coordinate descent algorithms highlighted in Algorithms \ref{alg:fullbcd} and \ref{alg:batchedbcd}. By exploiting the biconvexity of the problem, we can alternate over updating the $X$-variables and $W$-variables to train the network.

Note Algorithm \ref{alg:batchedbcd} is different from Algorithm \ref{alg:fullbcd} in three ways. First, re-initialization is required for the $X$-variables each time a new batch of data points are loaded. Second, the sub-problems for updating $W$-variables are different as shown in Section \ref{para:wbbatchsubprob}. Lastly, an additional parameter $K$ is introduced to specify the number of training alternations for each batch. Typically, we set $K = 1$.

\begin{figure}[t!]
    \vspace*{-\baselineskip}
    
    \begin{minipage}{\columnwidth}
    \begin{algorithm}[H]
	    \caption{Non-batched BCD Algorithm}
	    \label{alg:fullbcd}
	    \begin{algorithmic}[1]
		\STATE{Initialize $(W_{l})^L_{l=0}$.}
		\STATE{Initialize $X_0$ with input matrix $X$.}
		\STATE{Initialize $X_1, \hdots, X_L$ with neural network feed forward rule.}
		\REPEAT
		\STATE{$X_L \gets \arg\min_{Z} \: {\cal L}(Y, W_LZ) + \lambda B_{L-1}(Z, X_{L-1}^0)$}
		\FOR{$l = L-1, \hdots, 1$}
		\STATE{$X_l \leftarrow \arg\min_{Z} \: B_l(X_{l+1}, W_lZ)	+ B_{l-1}(Z, X_{l-1}^0)$}
		\ENDFOR
		\STATE{$W_L \leftarrow \arg\min_{W} \: {\cal L}(Y,WX_L) + \rho_l\pi_l(W)$}
		\FOR{$l = L-1, \hdots, 0$}
		\STATE{$W_l \leftarrow \arg\min_{W} \: \lambda B_l(X_{l+1},WX_l) + \rho_l\pi_l(W)$}
		\ENDFOR
		\UNTIL{convergence}
	\end{algorithmic}
\end{algorithm}
\end{minipage}

\begin{minipage}{\columnwidth}
\begin{algorithm}[H]
	\caption{Batched BCD Algorithm}
	\label{alg:batchedbcd}
	\begin{algorithmic}[1]
		\STATE{Initialize $(W_{l})^L_{l=0}$.}
		\REPEAT
		\STATE{$(W_{l}^0)^L_{l=0} \leftarrow (W_{l})^L_{l=0}$}
		\STATE{Re-initialize $X_0$ with a batch sampled from input matrix $X$.}
		\STATE{Re-initialize $X_1, \hdots, X_L$ with neural network feed forward rule.}
		\FOR{alternation$ = 1, \hdots, K$}
		\STATE{$X_L \leftarrow \arg\min_{Z} \: {\cal L}(Y, W_LZ) + \lambda B_{L-1}(Z, X_{L-1}^0)$}
		\FOR{$l = L-1, \hdots, 1$}
		\STATE{$X_l \leftarrow \arg\min_{Z} \: \lambda B_l(X_{l+1}, W_lZ)	+ \lambda B_{l-1}(Z, X_{l-1}^0)$}
		\ENDFOR
		\STATE{$W_L \leftarrow \arg\min_{W} \: {\cal L}(Y,WX_L) + \rho_L\pi_L(W) + \gamma_l\|W - W_{L}^0\|_F^2$}
		\FOR{$l = L-1, \hdots, 0$}
		\STATE{$W_l \leftarrow \arg\min_{W} \: \lambda B_l(X_{l+1}, WX_l) + \rho_l\pi_l(W) + \gamma_l\|W - W_{l}^0\|_F^2$}
		\ENDFOR
		\ENDFOR
		\UNTIL{convergence}
	\end{algorithmic}
\end{algorithm}
\end{minipage}

\end{figure}

\subsubsection{Updating $X$-variables}
For fixed $W$-variables, the problem of updating $X$-variables can be solved by cyclically optimizing $X_l, ~~ l = 1, \hdots, L, $ with $(X_j)_{j \ne l}$ fixed. We initialize our $X$-variables by feeding forward through the equivalent neural network and update the $X_l$'s backward from $X_L$ to $X_1$ in the spirit of backpropagation.

We can derive the sub-problem for $X_l, ~~l=1, \hdots, L-1$ with $(X_j)_{j \ne l}$ fixed from \eqref{fnn}. The sub-problem writes
\begin{align} \label{fullxsubproblem}
	X_l^+ = \arg\min_{Z} & \:  B_l(X_{l+1}, W_lZ ) + B_{l-1}(Z, X_{l-1}^0)
\end{align}

where $X_{l-1}^0 := W_{l-1}X_{l-1}$. By construction, the sub-problem \eqref{fullxsubproblem} is convex and parallelizable across data points. Note in particular when our activation is a ReLU, the objective function in~\eqref{fullxsubproblem} is in fact strongly convex and has a continuous first derivative.

For the last layer (i.e., $l = L$), the sub-problem derived from \eqref{fnn} writes differently
\begin{align} \label{fullxsubproblemlast}
	X_L^+ = \arg\min_{Z} & \: {\cal L}(Y, W_LZ) +\lambda B_{L-1}(Z, X_{L-1}^0)  
\end{align}

where $X_{L-1}^0 := W_{L-1}X_{L-1}$. For common losses such as mean square error (MSE) and cross-entropy, the subproblem is convex and parallelizable across data points. Specifically, when the loss is MSE and we use a ReLU activation at the layer before the output layer, \eqref{fullxsubproblemlast} becomes
\begin{align}\label{eq:xll}
	X_L^+ = \arg\min_{Z \ge 0} & \: ||Y - W_LZ||^2_F + \dfrac{\lambda}{2} ||Z - X_{L-1}^0||^2_F \notag
\end{align}

where $X_{L-1}^0 := W_{L-1}X_{L-1}$ and we use the fact that $X_{L-1}^0$ is a constant to equivalently replace $B_{L-1}$ as in \eqref{eq:reluBC} by a squared Frobenius term. The sub-problem is a non-negative least squares for which specialized methods exist \cite{kim2014algorithms}.

For a cross-entropy loss and when the second-to-last layer is a ReLU activation, the sub-problem for the last layer takes the convex form
\begin{align}
	X_L^+ = &\arg\min_{Z \ge 0} -\Tr Y^\top \log s(W_LZ) +\notag\\&\dfrac{\lambda}{2} ||Z - X_{L-1}^0||^2_F ,
\end{align}
where $s(\cdot): \mathbb{R}^n \rightarrow \mathbb{R}^n$ is the softmax function and $\log$ is the element-wise logarithm. \citet{askari2018} show how to solve the above problem using bisection.


\subsubsection{Updating $W$-variables}
With fixed $X$-variables, the problem of updating the $W$-variables can be solved in parallel across layers and data points. 
\paragraph{Sub-problems for non-batched model.} The problem of updating $W_l$ at intermediate layers becomes
\begin{align} \label{fullwbsubproblem}
	W_l = \arg\min_{W,} & \: \lambda B_l(X_{l+1},WX_l) + \rho_l\pi_l(W).
\end{align}
Again, by construction, the sub-problem \eqref{fullwbsubproblem} is convex and parallelizable across data points. Also, since there is no coupling in the $W$-variables between layers, the sub-problem \eqref{fullwbsubproblem} is parallelizable across layers.

For the last layer, the sub-problem becomes
\begin{align} \label{fullwbsubproblemlast}
	W_L = \arg\min_{W} & \: {\cal L}(Y,WX_L) + \rho_L\pi_L(W).
\end{align}

\paragraph{Sub-problems for batched model.} \label{para:wbbatchsubprob}
As shown in Section \ref{sec:batchedflnn}, the introduction of regularization terms between $W$ and values from a previous batch require the sub-problems (\ref{fullwbsubproblem}, \ref{fullwbsubproblemlast}) be modified. \eqref{fullwbsubproblem} now becomes
\begin{align} \label{batchedwbsubproblem}
	W_l = \arg\min_{W}  \: \lambda & B_l(X_{l+1}, WX_l) + \rho_l\pi_l(W) \notag \\
	+ &\gamma_l\|W - W_{l}^0\|_F^2 ,
\end{align}
while \eqref{fullwbsubproblemlast} becomes
\begin{align} \label{batchedwbsubproblemlast}
    W_L = \arg\min_{W}  \: {\cal L}&(Y,WX_L) + \rho_L\pi_L(W) \notag \\
	+ &\gamma_L\|W - W_{L}^0\|_F^2  .
\end{align}
Note that these sub-problems in the case of a ReLU activation are strongly convex and parallelizable across layers and data points.

%


\section{Numerical Experiments}\label{sec:exp}
In this section, we compare Fenchel lifted networks against other lifted models discussed in the introduction and against traditional neural networks. In particular, we compare our model against the models proposed by \citet{taylor2016training}, \citet{lau2018proximal} and \citet{askari2018} on MNIST. Then we compare Fenchel lifted networks against a fully connected neural network and LeNet-5 \citep{lecun1998gradient} on MNIST. Finally, we compare Fenchel lifted networks against LeNet-5 on CIFAR-10. For a discussion on hyperparameters and how model paramters were selected, see Appendix \ref{apx:hyperparam}.

\subsection{Fenchel lifted networks vs.\ lifted models}
Here, we compare the non-batched Fenchel lifted network against the models proposed by \citet{taylor2016training}\footnote{Code available in \url{https://github.com/PotatoThanh/ADMM-NeuralNetworks}}, \citet{lau2018proximal}\footnote{Code available in \url{https://github.com/deeplearning-math/bcd_dnn}} and \citet{askari2018}. The former model is trained using ADMM and the latter ones using the BCD algorithms proposed in the respective papers. In Figure \ref{fig:bcd}, we compare these models on MNIST with a 784-300-10 architecture (inspired by \citet{lecun1998gradient}) using a mean square error (MSE) loss. 

\begin{figure}[h!]
    \centering
    \includegraphics[width=1\columnwidth]{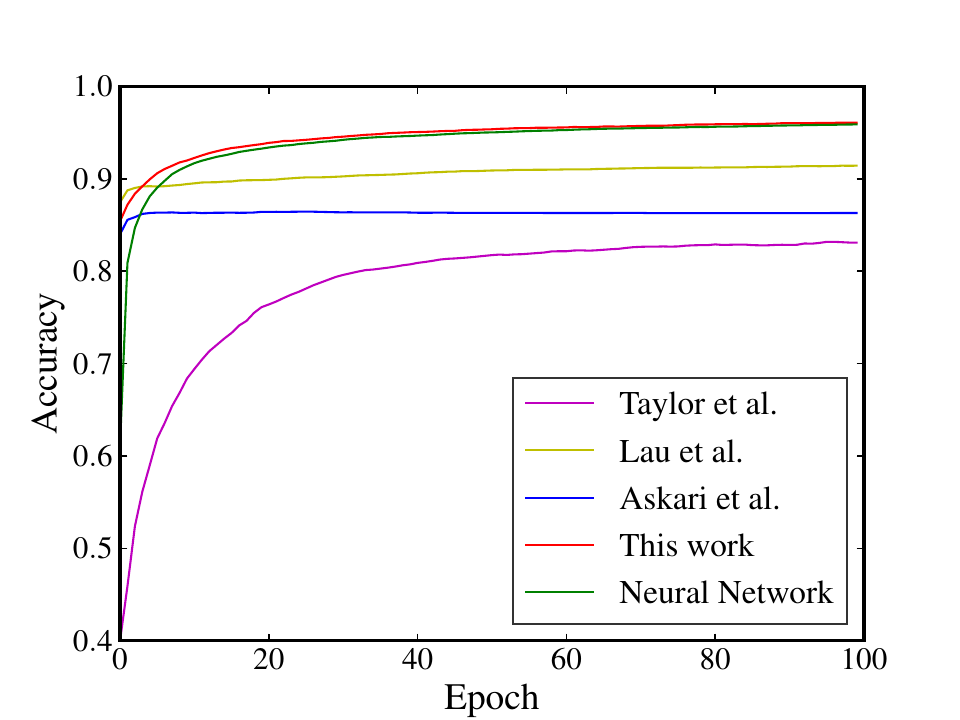}
    \caption{Test set performance of different lifted methods with a 784-300-10 network architecture on MNIST with a MSE loss.  Final test set performances: \textbf{Taylor et al.} 0.834, \textbf{Lau et al.} 0.914, \textbf{Askari et al.} 0.863,  \textbf{Neural Network} 0.957, \textbf{This work} 0.961.}
    \label{fig:bcd}
\end{figure}

After multiple iterations of hyperparameter search with little improvement over the base model, we chose to keep the hyperparameters for \citet{taylor2016training} and \citet{lau2018proximal} as given in the code. The hyperparameters for \citet{askari2018} were tuned using cross validation on a hold-out set during training. Our model used these same parameters and cross validated the remaining hyperparameters. The neural network model was trained using SGD. The resulting curve of the neural network is smoothed in Figure \ref{fig:bcd} for visual clarity. From Figure \ref{fig:bcd} it is clear that Fenchel lifted networks vastly outperform other lifted models and achieve a test set accuracy on par with traditional networks.


\subsection{Fenchel lifted networks vs.\ neural networks on MNIST}
For the same 784-300-10 architecture as the previous section, we compare the batched Fenchel lifted networks against traditional neural networks trained using first order methods. We use a cross entropy loss in the final layer for both models. The hyperparameters for our model are tuned using cross validation. Figure \ref{fig:mlp} shows the results.

\begin{figure}[h!]
\centering
\includegraphics[width=1\columnwidth]{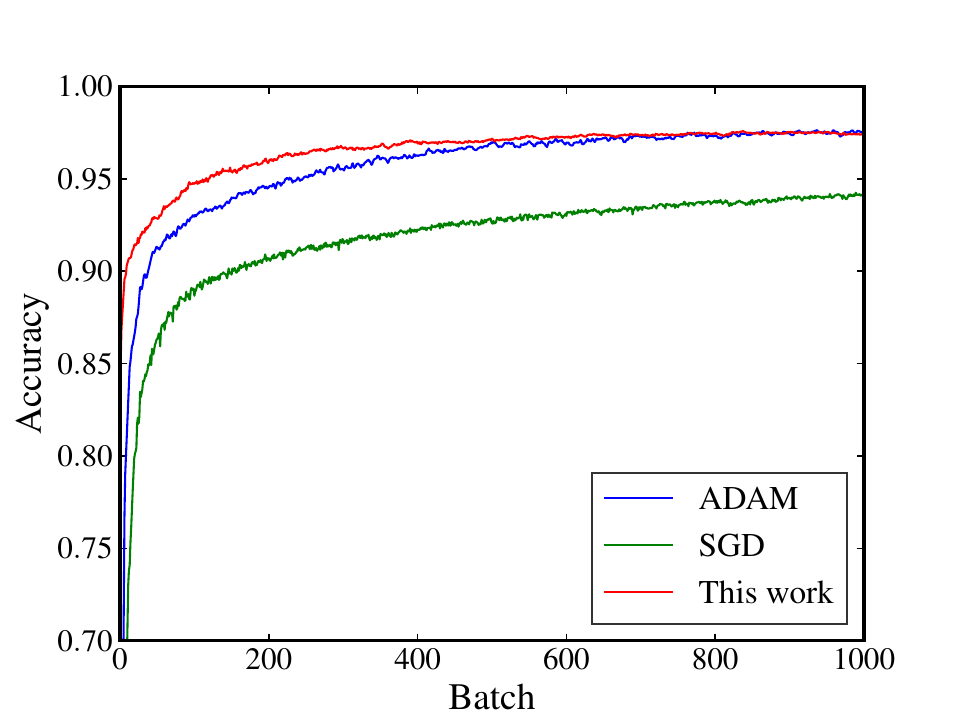}
\caption{Test set performance of Fenchel lifted networks and fully connected networks trained using Adam and SGD on a 784-300-10 network architecture on MNIST with cross entropy loss. Total training time was 10 epochs. Final test set performances: \textbf{SGD} 0.943, \textbf{Adam} 0.976, \textbf{This work} 0.976.}
\label{fig:mlp}
\end{figure}

As shown in Figure \ref{fig:mlp}, Fenchel lifted networks learn faster than traditional networks as shown by the red curve being consistently above the blue and green curve. Although not shown, between batch 600 and 1000, the accuracy on a training batch would consistently hit 100\% accuracy. The advantage of the Fenchel lifted networks is clear in the early stages of training, while towards the end the test set accuracy and the accuracy of an Adam-trained network converge to the same values.

We also compare Fenchel lifted networks against a LeNet-5 convolutional neural network on MNIST. The network architecture is 2 convolutional layers followed by 3 fully-connected layers and a cross entropy loss on the last layer. We use ReLU activations and average pooling in our implementation. Figure \ref{flnnadam} plots the test set accuracy for the different models.

\begin{figure}[h!]
\centering
\includegraphics[width=1\columnwidth]{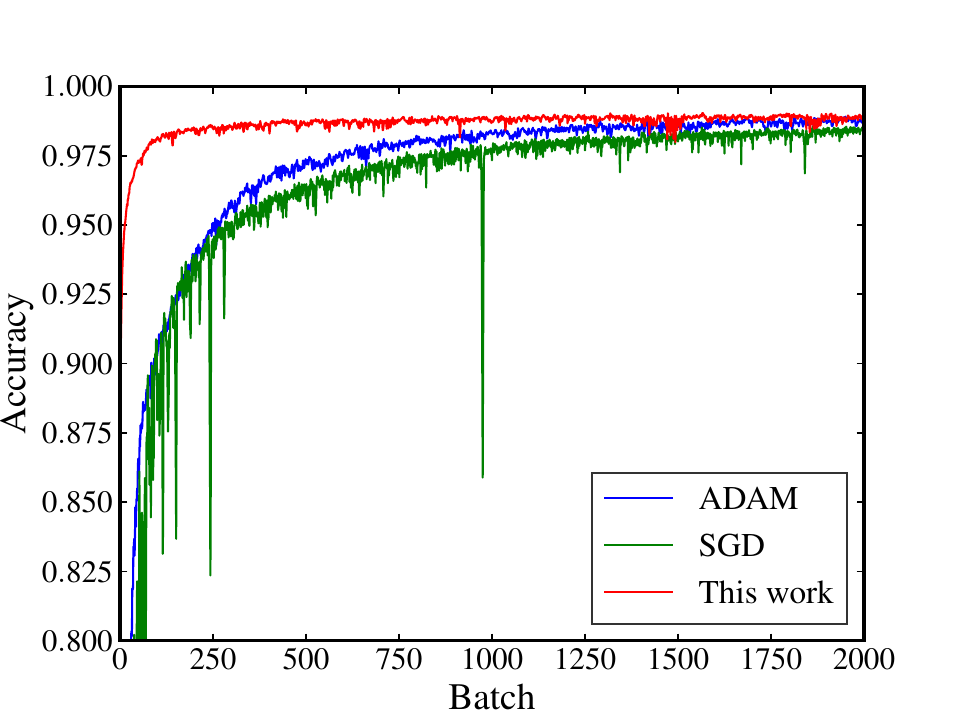}
\caption{Test set performance of Fenchel lifted networks and LeNet-5 trained using Adam and SGD on MNIST with a cross entropy loss. Total training time was 20 epochs. Final test set performances: \textbf{SGD} 0.986, \textbf{Adam} 0.989, \textbf{This work} 0.990.}
\label{flnnadam}
\end{figure}

In Figure \ref{flnnadam}, our method is able to nearly converge to its final test set accuracy after only 2 epochs while Adam and SGD need the full 20 epochs to converge. Furthermore, after the first few batches, our model is attaining over 90\% accuracy on the test set while the other methods are only at 80\%, indicating that our model is doing something different (in a positive way) compared to traditional networks, giving them a clear advantage in test set accuracy.

\subsection{Fenchel lifted networks vs CNN on CIFAR-10}
In this section, we compare the LeNet-5 architechture and with Fenchel lifted networks on CIFAR-10. Figure \ref{fig:cifar10} compares the accuracies of the different models.

\begin{figure}[h!]
\centering
\includegraphics[width=1\columnwidth]{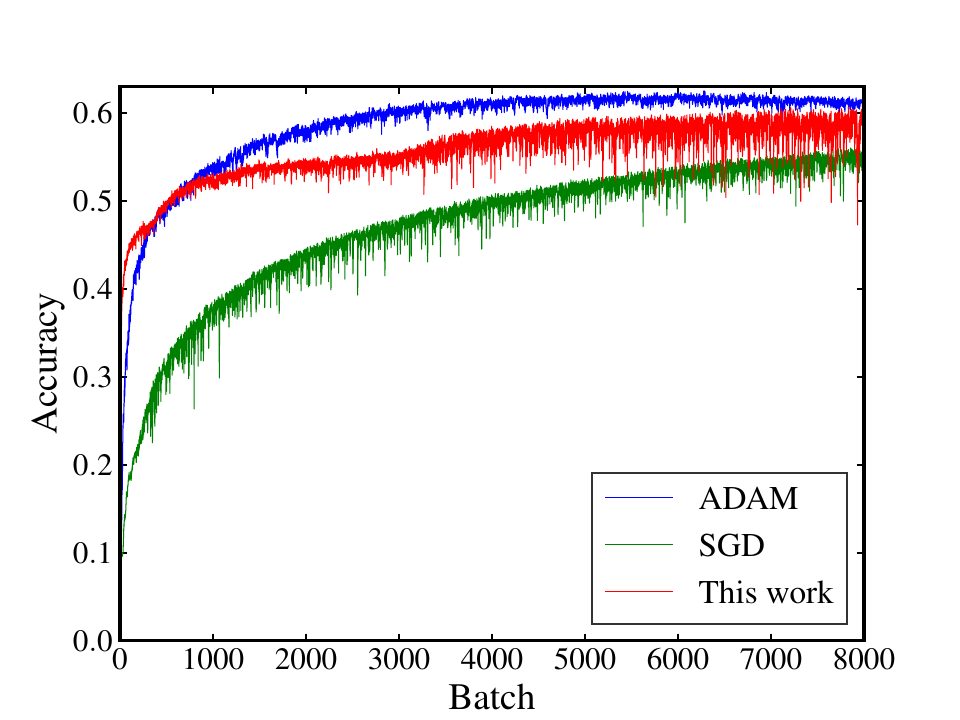}
\caption{Test set performance of Fenchel lifted networks and LeNet-5 trained using Adam and SGD on CIFAR-10 with a cross entropy loss. Total training time was 80 epochs. Final test set performance: \textbf{SGD} 0.565, \textbf{Adam} 0.625, \textbf{This work} 0.606}
\label{fig:cifar10}
\end{figure}

In this case, the Fenchel lifted network still outperforms the SGD trained network and only slightly under performs compared to the Adam trained network. The larger variability in the accuracy per batch for our model can be attributed to the fact that in this experiment, when updating the $W$-variables, we would only take one gradient step instead of solving \eqref{batchedwbsubproblem} and \eqref{batchedwbsubproblemlast} to completion. We did this because we found empirically solving those respective sub-problems to completion would lead to poor performance at test time.

\section{Conclusion and Future Work}
In this paper we propose Fenchel lifted networks, a family of models that provide a rigorous lower bound of the traditional neural network training problem. Fenchel lifted networks are similar to other methods that lift the dimension of the training problem, and thus exhibit many desirable properties in terms of scalability and the parallel structure of its sub-problems. As a result, we show that our family of models can be trained efficiently using block coordinate descent where the sub-problems can be parallelized across data points and/or layers. Unlike other similar lifted methods, Fenchel lifted networks are able to compete with traditional fully connected and convolutional neural networks on standard classification data sets, and in some cases are able to outperform them.

Future work will look at extending the ideas presented here to Recurrent Neural Networks, as well as exploring how to use the class of models described in the paper to train deeper networks. 



\bibliographystyle{apalike}
\bibliography{ref}

\clearpage
\appendix
\onecolumn
\gdef\thesection{Appendix \Alph{section}}

{\centering{{\LARGE\bfseries Supplementary material}}}

\section{Variable Scaling} \label{apx:varscale}
Note that the new model \eqref{fnnp} has introduced $L+1$ more hyperparameters. We can use variable scaling and the dual formulation to show how to effectively reduce this to only \textit{one} hyperparameter. Consider the model with ReLU activations, that is, the biconvex function as in \eqref{eq:reluBC} and regularization functions $\pi_l(W_l) = \|W_l\|_F^2$ for $l = 0,\hdots,L$. Note that $B_\phi$ is homogeneous of degree 2, that is for any $U,V$ and $\gamma$ we have 
\begin{align*}
    \gamma B_\phi(V,U) = B_\phi(\sqrt{\gamma} V, \sqrt{\gamma} U)
\end{align*}
Define $\lambda_{-1} = 1$ and the scalings
\begin{align*}
    \Bar{X}_l := \sqrt{\lambda_{l-1}} X_l, \; \; \Bar{W}_l := \sqrt{\dfrac{\lambda_l}{\lambda_{l-1}}} W_l, \;\; 
\end{align*}
Then \eqref{fnnp} becomes
\begin{align} \label{eq:fn_dual_scale}
    G(\lambda) := &\min_{(\bar{W}_l)_{l=0}^{L}, (\bar{X}_l)_1^{L+1}} \: {\cal L}(Y,\sqrt{\lambda_L}( \bar{W}_L\bar{X}_L)) \notag \\
	&+ \sum^L_{l=0}\rho_l 
	\pi_l(\sqrt{\dfrac{\lambda_{l-1}}{\lambda_l}} W_l)
	+ \sum_{l=0}^{L-1} B_l(\bar{X}_{l+1}, 
	\bar{W}_l \bar{X}_l) \notag \\
	\mbox{s.t.} & \: \bar{X}_0 = X, \; \Bar{X}_l \geq 0, \; l = 0,\hdots,L
\end{align}

Using the fact $\pi_l(W_l) = \|W_l\|_F^2$ and defining $\bar{\rho}_l = \rho_l \dfrac{\lambda_{l-1}}{\lambda_l}$ we have

\begin{align} \label{eq:fn_dual_scale2}
    G(\lambda) := &\min_{(\bar{W}_l)_{l=0}^{L}, (\bar{X}_l)_1^{L+1}} \: {\cal L}(Y,\sqrt{\lambda_L}( \bar{W}_L\bar{X}_L)) \notag \\
	&+ \sum^L_{l=0}
	\bar{\rho}_l \|\bar{W}_l\|_F^2
	+ \sum_{l=0}^{L-1} B_l(\bar{X}_{l+1}, 
	\bar{W}_l \bar{X}_l) \notag \\
	\mbox{s.t.} & \: \bar{X}_0 = X, \; \Bar{X}_l \geq 0, \; l = 0,\hdots,L
\end{align}

where $G(\lambda)$ is now only a function of one variable $\lambda_L$ as opposed to $L$ variables. Note that this argument for variable scaling still works when we use average pooling or convolution operations in conjunction with a ReLU activation since they are linear operations. Note furthermore that the same scaling argument works in place of any norm due to the homogeneity of norms -- the only thing that would change is how $\bar{\rho}$ is scaled by $\lambda_{l-1}$ and $\lambda_l$.

Another way to show that we only require one hyperparameter $\lambda$ is to note the equivalence
\begin{align*}
    B_l (v,u) \leq 0 \;\;  \forall l \Longleftrightarrow \sum_{l} B_l (v,u) \leq 0
\end{align*}
Then we may replace the $L$ biconvex constraints in \eqref{fnn} by the equivalent constraint $\sum_{l} B_l (v,u) \leq 0$. Since this is only one constraint, when we dualize we only introduce \emph{one} Lagrange multiplier $\lambda$.

\section{One-layer Regression Setting} \label{apx:global}
In this section, we show that for a one layer network we are able to convert a non-convex optimization problem into a convex one by using the BC condition described in the main text.\\

Consider a regression setting where $Y = \phi(W^\ast X)$ for some fixed $W^\ast \in \mathbb{R}^{p\times n}$ and a given data matrix $X \in \mathbb{R}^{n \times m}$. Given a training set $(X,Y)$ we can solve for $W$ by solving the following non-convex problem

\begin{equation} \label{eq:ls}
    \min_W ~ \|Y - \phi(W X)\|_F^2.
\end{equation}

We could also solve the following relaxation of \eqref{eq:ls} based on the BC condition

\begin{align} \label{eq:cvx}
    \min_W ~ B_\phi(Y,WX)
\end{align}

Note \eqref{eq:cvx} is trivially convex in $W$ by definition of $B_\phi(\cdot,\cdot)$. Furthermore, by construction $B_\phi(Y,WX) \geq 0 $ and $B_\phi(Y,WX) = 0$ if and only if $Y = \phi(WX)$. Since $Y = \phi(W^\ast X)$, it follows $W^\ast$ (which is the minimizer of \eqref{eq:ls}) is a global minimizer of the convex program \eqref{eq:cvx}. Therefore, we can solve the original non-convex problem \eqref{eq:ls} to global optimality by instead solving the convex problem presented in \eqref{eq:cvx}.

\section{Hyperparameters for Experiments}
\label{apx:hyperparam}
For all experiments that used batching, the batch size was fixed at 500 and $K=1$. We observed empirically that larger batch sizes improved the performance of the lifted models. To speed up computations, we set $K=1$ and empirically find this does not affect final test set performance. For batched models, we do not use $\pi_l(\cdot)$ since we explicitly regularize through batching (see \eqref{bfnnp}) while for the non-batched models we set $\pi_l(W_l) = \|W_l\|_F^2$ for all $l$. For models trained using Adam, the learning rate was set to $\eta = 10^{-3}$ and for models trained using SGD, the learning rate was set to $\eta = 10^{-2}$. The learning rates were a hyperparamter that we picked from \{$10^{-1}, 10^{-2}, 10^{-3}, 10^{-4} $\} to give the best final test performance for both Adam and SGD. \\

For the network architechtures described in the experimental results, we used the following hyperparamters:

\begin{itemize}
    \item Fenchel Lifted Network for LeNet-5 architecture
    \begin{enumerate}
        \item $\rho_1 = 1e-4, \lambda_1 = 5$
        \item $\rho_2 = 1e-2, \lambda_2 = 5$
        \item $\rho_3 = 1, \;\;\;\;\;\;\;\; \lambda_3 = 1$
        \item $\rho_4 = 1, \;\;\;\;\;\;\;\; \lambda_4 = 1$
        \item $\rho_5 = 1$
    \end{enumerate}
    \item Fenchel Lifted Network for 784-300-10 architecture (batched)
    \begin{enumerate}
        \item $\rho_1 = 1, \;\;\;\;\;\;\;\;\; \lambda_1 = 0.1$
        \item $\rho_2 = 100$
    \end{enumerate}
    \item Fenchel Lifted Network for 784-300-10 architecture (non-batched)
    \begin{enumerate}
        \item $\rho_1 = 1e-2, \; \lambda_1 = 0.1$
        \item $\rho_2 = 10$
    \end{enumerate}
\end{itemize}

For all weights the initialization is done through Xavier initialization implemented in TensorFlow. The $\rho$ variables are chosen to balance the change of variables across layers in iterations. Although the theory in Appendix A states we can collapse all $\lambda$ hyperparameters into a single hyperparameter, due to time constraints, we were unable to implement this change upon submission. We also stress that the hyperparamter search over the $\rho$'s were very coarse and a variety of $\rho$ values worked well in practice; for simplicitly we only present the ones we used to produce the plots in the experimental results.

\end{document}